\documentclass[lettersize,journal]{IEEEtran}
\usepackage{amsmath,amsfonts}
\usepackage{algorithmic}
\usepackage{algorithm}
\usepackage{array}
\usepackage[caption=false,font=normalsize,labelfont=sf,textfont=sf]{subfig}
\usepackage{textcomp}
\usepackage{stfloats}
\usepackage{url}
\usepackage{verbatim}
\usepackage{graphicx}
\usepackage{cite}
\usepackage{amssymb}
\usepackage{mathtools}
\usepackage{amsthm}
\usepackage{bbding}
\usepackage{diagbox}
\usepackage{multirow}
\usepackage{color}
\usepackage[utf8]{inputenc} 
\usepackage[T1]{fontenc}    
\usepackage{booktabs}       
\usepackage{nicefrac}       
\usepackage{microtype}      
\usepackage{xcolor}         

\usepackage{newfloat}
\usepackage{listings}
\DeclareCaptionStyle{ruled}{labelfont=normalfont,labelsep=colon,strut=off} 
\DeclareSubrefFormat{myparens}{#1(#2)}
\floatstyle{ruled}
\newfloat{listing}{tb}{lst}{}
\floatname{listing}{Listing}
\newtheorem{proposition}{Proposition}

\begin{document}

\title{Rethinking Residual Connection in Training Large-Scale Spiking Neural Networks}

\author{Yudong Li, Yunlin Lei, Xu Yang
\thanks{Yudong Li, Yunlin Lei, and Xu Yang are with School of Computer Science and Technology, Beijing Institute of Technology, Beijing 100081, China}
\thanks{Corresponding author: Xu Yang}}

\markboth{}
{}


\maketitle

\begin{abstract}
Spiking Neural Network (SNN) is known as the most famous
brain-inspired model, but the non-differentiable spiking mechanism makes it hard to train
large-scale SNNs. To facilitate the training
of large-scale SNNs, many training methods are borrowed from 
Artificial Neural Networks (ANNs), among which deep residual
learning is the most commonly used.
But the unique features of SNNs make prior intuition built upon
ANNs not available for SNNs. Although there are a few studies
that have made some pioneer attempts on the topology of Spiking ResNet,
the advantages of different connections remain unclear.
To tackle this issue, we analyze the merits and
limitations of various residual connections and empirically
demonstrate our ideas with extensive experiments.
Then, based on our observations, we
abstract the best-performing connections into densely additive (DA) connection,
extend such a concept to other topologies, and propose four architectures
for training large-scale SNNs, termed DANet, which brings up to
13.24\% accuracy gain on ImageNet.
Besides, in order to present a detailed methodology
for designing the topology of large-scale SNNs,
we further conduct in-depth discussions on
their applicable scenarios in terms of their performance
on various scales of datasets
and demonstrate their advantages over prior architectures.
At a low training expense,
our best-performing ResNet-50/101/152
obtain 73.71\%/76.13\%/77.22\% top-1 accuracy on ImageNet with 4 time steps.
We believe that this work shall give more insights for future works to design
the topology of their networks and promote the development of large-scale SNNs.
The code will be publicly available.
\end{abstract}


\section{Introduction}
\IEEEPARstart{S}{piking} Neural Networks (SNNs), the third generation of
Artificial Neural Networks (ANNs) \cite{maass1997networks}, are considered a promising
bionic model and have made great progress over the last few years
\cite{li2022neuromorphic,kim2021revisiting,fang2021incorporating,rathi2021diet}.
Nevertheless, the non-differentiable spiking
mechanism makes it hard to train large-scale SNNs.
In general, there are two main routes to obtain large-scale SNNs.
The first is to convert a
pre-trained ANN to its SNN version
\cite{hu2018spiking, han2020rmp, sengupta2019going}. The activation values of an ANN are
approximated by the average firing rates of the corresponding SNN.
The SNNs obtained by this method can achieve
almost lossless accuracy, but they demand
long simulation time steps to reach a sufficiently
accurate estimation. The second method is to directly train a SNN
with approximated gradients and our study falls under this section.
There are a lot of methods to approximate gradients~\cite{xiao2021training,meng2022training,neftci2019surrogate},
among which the surrogate gradient (SG) method is the most widely
used~\cite{neftci2019surrogate}. The SG method keeps the original forward process unchanged
and adopts a surrogate function
to replace the non-differentiable spiking function
in the computational graph.
The SNNs obtained by the second method demand only a few time steps to
reach performance saturation but their performance is often worse than their ANN
counterparts.
In a nutshell, the above methods have their own merits and limitations
and there are a lot of outstanding works devoted to improving them
\cite{xiao2021training,fang2021deep,yao2021temporal,kim2022exploring},
most of which are borrowed from the mature
training methods of ANNs~\cite{he2016deep,he2016identity,bai2019deep,bai2020multiscale,hu2018squeeze,frankle2018lottery}.

Among these methods, the most frequently used approach is
deep residual learning~\cite{he2016deep}, which has lots of variants
\cite{he2016deep,he2016identity,he2019bag}.
However, the spatio-temporal characteristics of spiking neurons (SNs) make
SNNs critical to the selection of these variants and prior
intuition built upon ANNs is not available for SNNs anymore. Although
there are a few studies that have made some pioneer attempts on the topology of
Spiking ResNet~\cite{zheng2021going,fang2021deep,meng2022training,hu2021advancing},
it is still unclear which one performs better.
Thus, we analyze the
merits and limitations of various residual connections.
We found that, although the intrinsic sparsity of SNNs brings the energy-efficient property for SNNs,
it also results in sparse updates during
backpropagation, which hinders the learning of SNNs.
And minor modifications on the topologies will greatly affect the spiking
patterns of SNNs, which are of great importance to the flow of gradients and can greatly
enhance the performance of SNNs.
With extensive experiments, we empirically demonstrate our analysis and
conclude the properties of different residual connections.

Furthermore, since the discussed residual connections can be generally classified
into two classes depending on the existence of interblock SNs
({\em i.e.} whether there are SNs outside residual connections),
it is natural to borrow from the better class to improve the performance of the other class.
And we found that the best-performing class ({\em i.e.} without interblock SNs) can be abstracted into
densely additive (DA) connection. When such a concept is extended to the connections
with interblock SNs, the issue caused by interblock SNs can be effectively alleviated
and we term the series of SNNs derived from DA connection as DANet.
To further present a detailed methodology for designing the topology of large-scale SNNs, we also conduct
in-depth discussions on the applicable scenarios of the DANet family in terms of their performance on various scales of
datasets. At a low training expense, our best performing ResNet-50/101/152 obtain 73.71\%/76.13\%/77.22\% top-1
accuracy on ImageNet with 4 time steps. Compared to the Spiking-ResNet/SEW-ResNet reported in~\cite{fang2021deep},
our corresponding architectures outperform
them by up to 13.24\%/7.96\% on ImageNet with about a quarter of their training cost.
The main contributions of this paper can be summarized as follows:
\begin{itemize}
    \item We analyze the merits and limitations of various residual connections and empirically demonstrate our analysis.
    \item We abstract the concept of densely additive connection, extend it to other topologies, and propose the DANet family, which obtains impressive performance on ImageNet.
    \item To present a detailed methodology for designing large-scale SNNs, we analyze the performance of the DANet family on various scales of datasets and demonstrate their advantages over prior architectures.
\end{itemize}

\section{Related Work}
Since the proposal of residual connection~\cite{he2016deep}, there have been a lot of variants and
applications that are improving it or utilizing
it~\cite{he2016identity,he2019bag,vaswani2017attention,dosovitskiy2020image}. And
residual connection is widely adopted for training large-scale SNNs as
well~\cite{zheng2021going,fang2021deep,meng2022training,hu2021advancing}.
\cite{zheng2021going} are the first to extend the depth of directly trained large-scale SNNs to 50
layers.
They adopt the dedicatedly designed Batch Normalization (BN) method to
help stabilize the firing rates of SNs and facilitate the training of deeper SNNs.
We argue that a
similar effect can be achieved by simple modifications on the
topology of SNNs with specific parameters. And our DANet-C/D can further
alleviate the negative effect brought by interblock SNs.
Then, \cite{fang2021deep} propose SEW-ResNet, which further extends the depth of SNNs
to 152 layers by modifying the topology of the original residual connection.
However, they only discuss two connections and neglect the influence of spiking patterns on the learning
of large-scale SNNs.
Concurrent with SEW, \cite{hu2021advancing} propose MS-ResNet,
whose topology is quite similar to our DANet-A.
But they adopt tdBN \cite{zheng2021going} for normalization,
which makes their firing pattern quite different from ours and
results in their inferior performance to our DANet-A.
As SEW, they do not compare their method with other connections
without interblock SNs either,
while we study the properties of various connections and
obtain higher accuracy on ImageNet with much less training cost.
There are also a few studies adopting other topologies or conducting neural
architecture search for SNNs~\cite{meng2022training,kim2022neural},
but they do not demonstrate the superiority of their architectures either intuitively or theoretically and
cannot provide instructions on how to design the topology of large-scale SNNs.

Besides, the idea of DANet is quite similar to the DenseNet~\cite{huang2017densely} in ANNs. However,
DenseNet adopts concatenation as its dense operation while we adopt addition. For ANNs, DA operation will
cause numerical instability and disrupt training. For SNNs, dense concatenation will concate
sparse binary feature maps and accumulate estimation error, which will greatly harm the performance of large-scale SNNs
as we demonstrated in Sec.~\ref{ablation:DO}. And since SNNs will inherently decay input, dense addition will not
result in numerical instability but facilitate the flow of gradients in large-scale SNNs.
Thus, the design of DANet is counterintuitive for ANNs
but well suited for SNNs.

\begin{figure*}[t]
  \centering
  \subfloat[]{
  \label{FP-A}
  \centering
  \includegraphics[height=0.9\columnwidth]{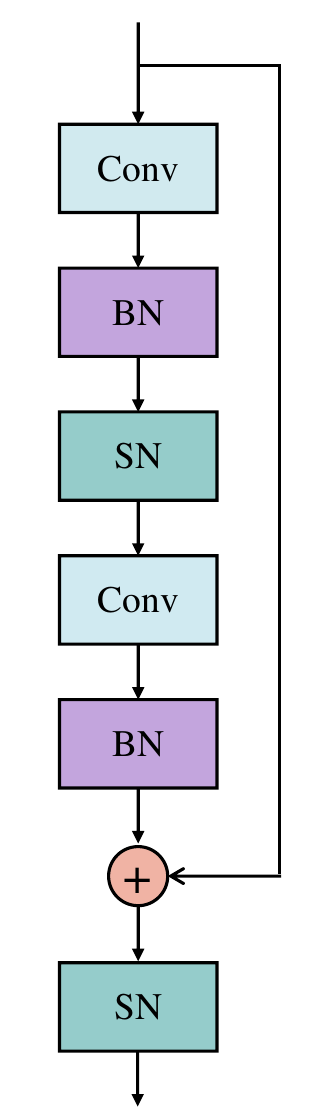}
  }
  \subfloat[]{
  \label{FP-B}
  \centering
  \includegraphics[height=0.9\columnwidth]{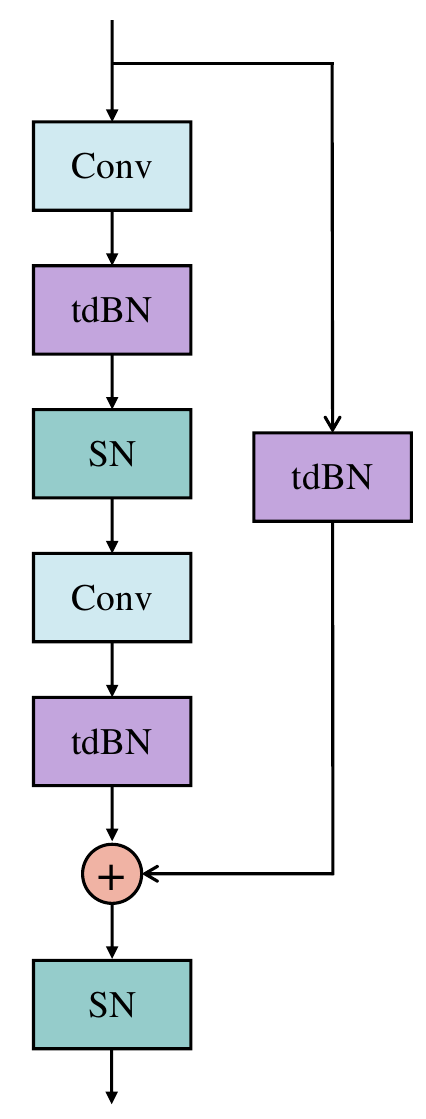}
  }
  \subfloat[]{
  \label{FP-C}
  \centering
  \includegraphics[height=0.9\columnwidth]{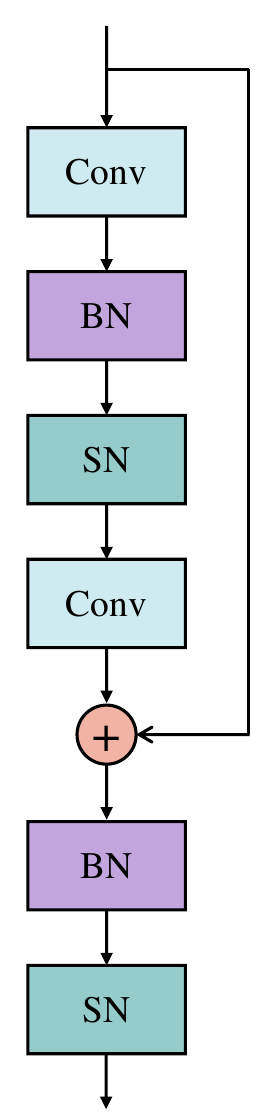}
  }
  \subfloat[]{
  \label{FP-D}
  \centering
  \includegraphics[height=0.9\columnwidth]{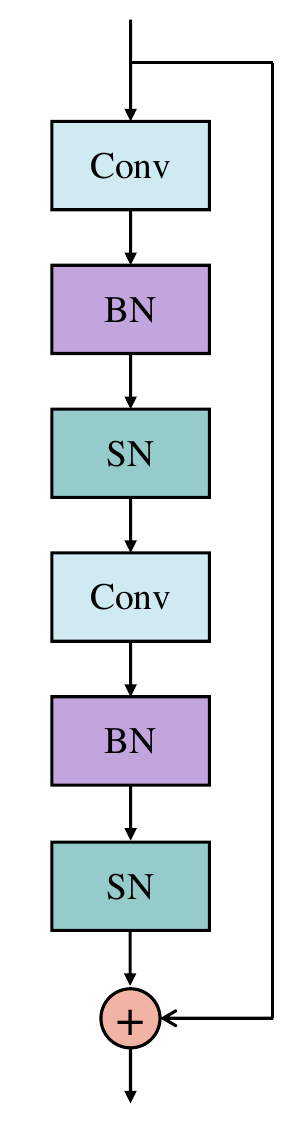}
  }
  \subfloat[]{
  \label{FP-E}
  \centering
  \includegraphics[height=0.9\columnwidth]{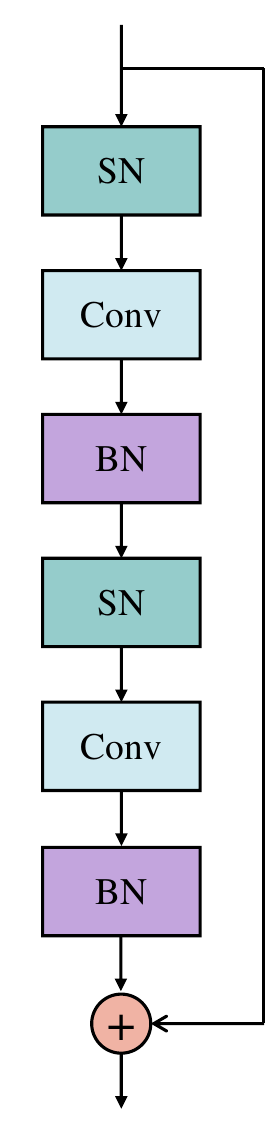}
  }
  \subfloat[]{
  \label{FP-F}
  \centering
  \includegraphics[height=0.9\columnwidth]{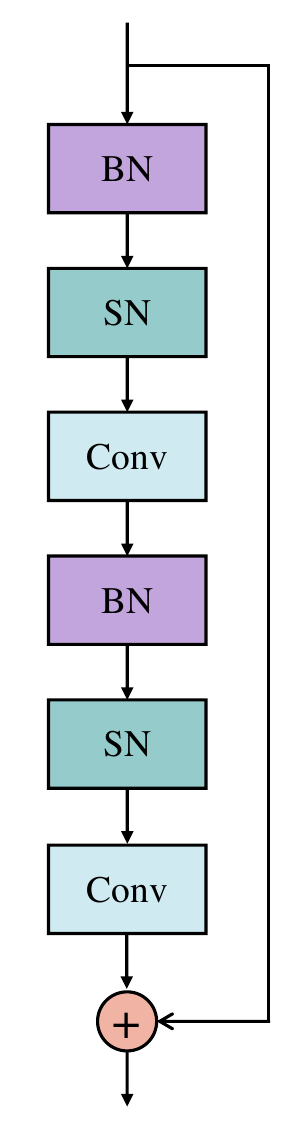}
  }
  \caption{Illustrations of different residual connections. (a) original. (b) tdBN. (c) BN after addition. (d) SEW. (e) PA-A. (f) PA-B.}
  \label{fig1}
\end{figure*}

\section{Preliminaries}
\subsection{Spiking Neuron Model}
Spiking Neuron (SN) is the basic computing unit of SNNs. The commonly used
integrate and fire (IF) and leaky integrate and fire (LIF) models \cite{burkitt2006review}
can be defined by the following equations:
\begin{equation}
    \label{eq2}
    \begin{gathered}
    V[t] = \left\{\begin{aligned}
    & (1 - \frac{1}{\tau}) \cdot U[t - 1] + \frac{1}{\tau} \cdot (X[t] + u_{rest}), & & LIF \\
    & U[t - 1] + X[t], & & IF
    \end{aligned}\right.
    \\
    S[t] = \Theta(V[t] - V_{th}),\\
    U[t] = V[t] \cdot (1-S[t]) + u_{rest} \cdot S[t],
    \end{gathered}
\end{equation}
where $t$ is the time step, $X$ is the input collected by synapses, $S$ is the generated spikes,
and $V$ and $U$ are the membrane potential before and after the generation
of spikes, respectively.
$V_{th}$ is the spiking threshold, $u_{rest}$ is the resting potential, and $\tau$ is
the membrane time constant, which are set as 1, 0, and 2, respectively.
$\Theta(x)$ is the Heaviside step function, which
outputs 1 when $x >= 0$, otherwise 0.
When $V \geq V_{th}$, the neuron will omit a spike and $U$ will be reset to $u_{rest}$.
For both IF and LIF models, $S[t]$ is the final output at time step $t$.

\subsection{Surrogate Function}
During backpropagation, we adopt the Sigmoid function as our surrogate function:
\begin{equation}
    \begin{gathered}
    \bar{\Theta}(x) = \frac{1}{1 + e^{- \alpha x}},\\
    \bar{\Theta}'(x) = \alpha \cdot \bar{\Theta}(x) \cdot (1-\bar{\Theta}(x)),
    \end{gathered}
\end{equation}
where $\alpha$ is the hyper-parameter that controls the slope of the
surrogate function. The bigger the $\alpha$, the steeper the function,
and we set $\alpha = 4$ in all the experiments.
The Heaviside function works fine during forward propagation and the surrogate function
takes its place in the computational graph during backpropagation. In
this way, we manage to perform backpropagation as ANNs do.

\section{Methodology}
\subsection{Residual Connection}
\label{method:ImF}
In this section, we will analyze six kinds of residual connections, as depicted in
Fig.~\ref{fig1}. Fig.~\subref*{FP-A} is the original residual connection
(abbreviated as Origin) in ANNs, but it often
obtains poor results in SNNs. \cite{zheng2021going}
propose tdBN as shown in Fig.~\subref*{FP-B}, which tries to force the
input distribution of each SN to be $N(0,V_{th}^2)$ when initialization.
We argue that the same effect
can be achieved by the connection shown in Fig.~\subref*{FP-C} with
a proper hyper-parameter setting,
which we abbreviate as BAA in the following. Since
each SN in BAA follows BN, its initial input distribution can be approximated by
$N(0, V_{th}^2)$ when $V_{th} = 1$, which is our default setting.
As tdBN can help stabilize training, BAA should have comparable performance
to tdBN. However, Origin, tdBN, and BAA are not suitable for training deep SNNs
according to {\bf Proposition~\ref{proposition1}}.

\begin{proposition}
    \label{proposition1}
    The interblock modules in residual connections tend to disrupt
    gradients, which make these connections not suitable for training
    deep SNNs.
\end{proposition}
\begin{proof}
  The proof of {\bf Proposition~\ref{proposition1}} mainly follows~\cite{fang2021deep}.
  For Origin, assuming that the SN will generate a spike
  whenever it receives a spike
  ({\em e.g.} IF model with $V_{th} = 1$) and the
  outputs from the last BN layer are always 0 ({\em i.e.}
  the identity mapping condition
  is met, $O^{l + k}[t]=\cdots=O^{l}[t]$,
  where $O^{l+k}[t]$ denotes the outputs from the
  $(l + k)$-th block at time step $t$),
  the gradients from the $l$-th block to
  the $(l + k)$-th block can be formulated as follows:
  \begin{equation}
      \label{eq4}
      \begin{gathered}
      O^{l + k}[t] = 0 + SN(O^{l + k - 1}[t])
      = \\ 0 + SN(SN(O^{l + k - 2}[t])) = \cdots,
      \\
      \frac{\partial O^{l + k}[t]}{\partial O^l[t]} =
        \prod_{i=1}^{i=k}\bar{\Theta}'(O^{l+i} - V_{th})
          \rightarrow \\ \left\{
              \begin{aligned}
              & +\infty, & & \bar{\Theta}'(O^{l+i} - V_{th}) > 1,\\
              & 1, & & \bar{\Theta}'(O^{l+i} - V_{th}) = 1,\\
              & 0, & & \bar{\Theta}'(O^{l+i} - V_{th}) < 1.
              \end{aligned} \right.
      \end{gathered}
  \end{equation}
  For most surrogate functions, $\bar{\Theta}'(O^{l+i} - V_{th}) = 1$
  is hard to satisfy. Since a surrogate function can be
  seen as a differentiable version of the Heaviside step function,
  when there is no spike, surrogate functions tend to produce small gradients.
  The intrinsic sparsity of SNNs further makes them easier
  to be small, {\em i.e.}
  vanishing gradients are often the case.
  And for tdBN and BAA, the additional BN layer cannot solve such an issue and
  the interblock modules still disrupt gradients and degrade the performance
  of deep SNNs.
\end{proof}

\begin{figure*}[t]
  \centering
  \subfloat[]{
  \label{LIF}
  \centering
  \includegraphics[height=0.75\columnwidth]{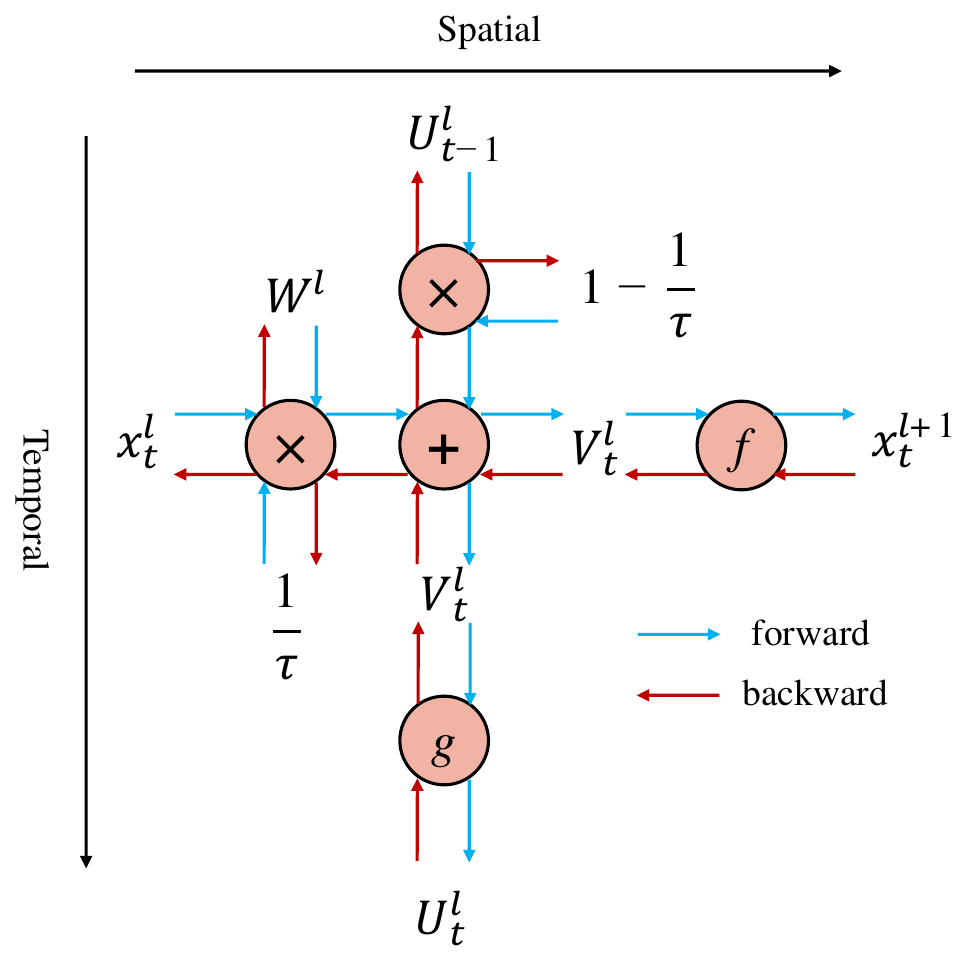}
  }
  \subfloat[]{
  \label{IF}
  \centering
  \includegraphics[height=0.75\columnwidth]{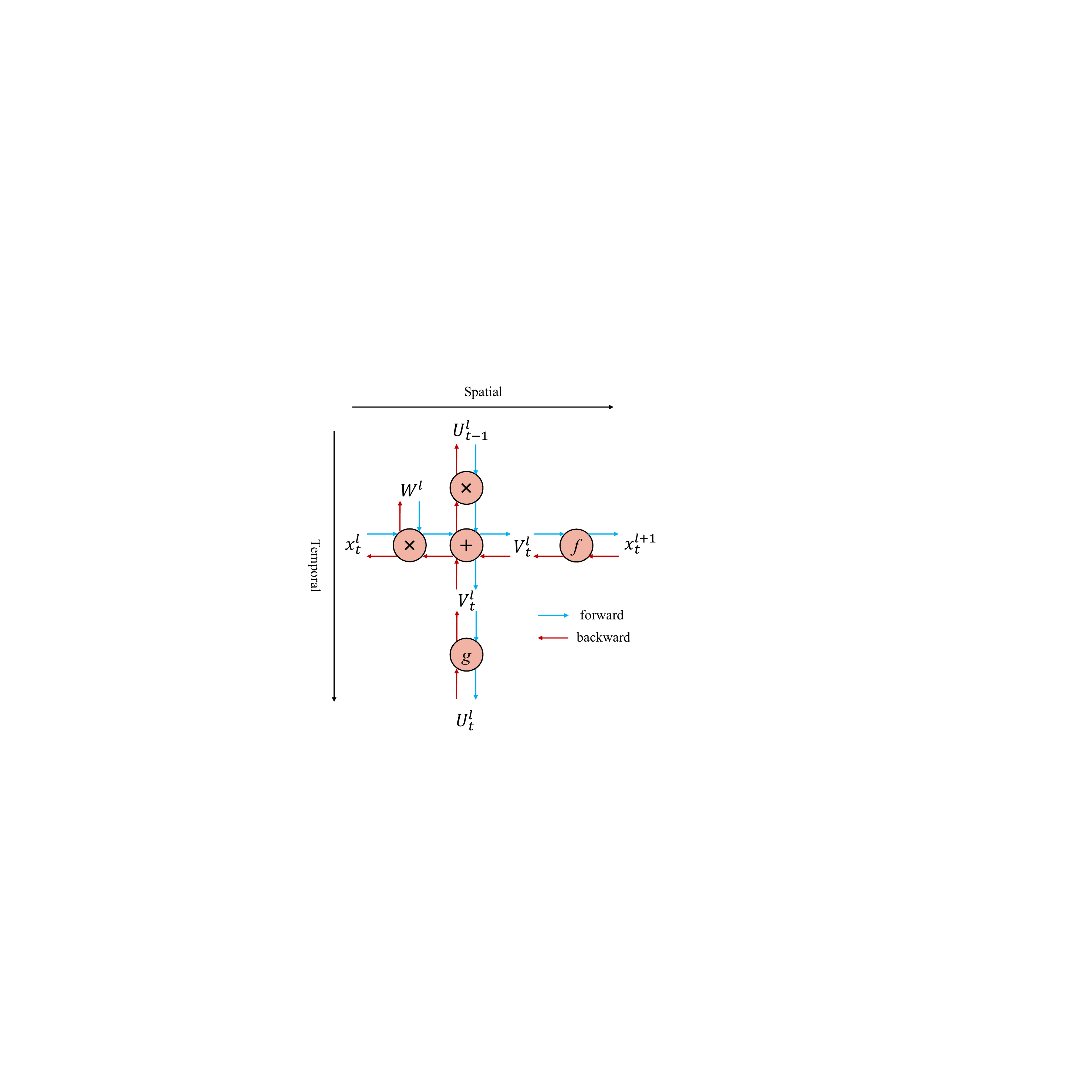}
  }
  \caption{Computational graphs of different neuron models. $f$ and $g$ denote
  the firing function and reset function respectively. (a) LIF. (b) IF.}
  \label{fig2}
\end{figure*}

\begin{figure*}[t]
  \centering
  \subfloat[]{
  \label{fr_res101}
  \centering
  \includegraphics[width=0.66\columnwidth]{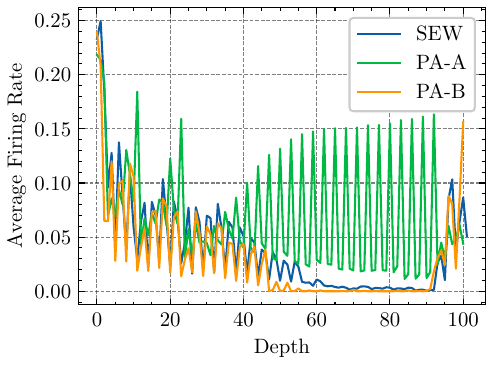}
  }
  \subfloat[]{
  \label{fr_res152}
  \centering
  \includegraphics[width=0.66\columnwidth]{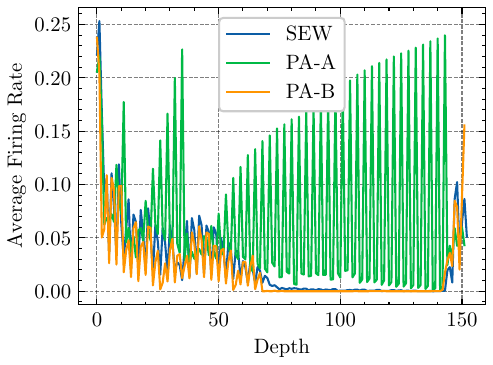}
  }
  \subfloat[]{
          \centering
          \includegraphics[width=0.66\columnwidth]{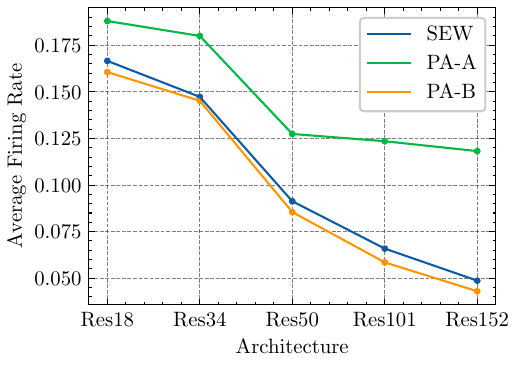}
          \label{fr_avg}
  }
  \caption{Firing rates of SEW, PA-A, and PA-B. (a)-(b) depict the layer-wise firing rates of different
  backbones and (c) demonstrates the overall average firing rates. (a) ResNet-101. (b) ResNet-152. (c) Average firing rates.}
  \label{fig3}
\end{figure*}

SEW, PA-A, and PA-B depicted in Fig.~\ref{fig1} are three improved connections.
The connection of SEW is the same as~\cite{fang2021deep} if we adopt addition as
spike-element-wise operation, which is the best-performing architecture reported in~\cite{fang2021deep}.
We term the latter two connections as PA-A and PA-B
since they are quite similar to
the pre-activation blocks in \cite{he2016identity}.
Additionally, PA-A has a similar topology as MS-ResNet \cite{hu2021advancing},
except that MS-ResNet adopts tdBN. But such a minor modification greatly changes
its spiking pattern and final performance as we will discuss in the following.

\begin{figure*}[t]
  \centering
  \subfloat[]{
  \label{inspiration}
  \centering
  \includegraphics[height=1.0\columnwidth]{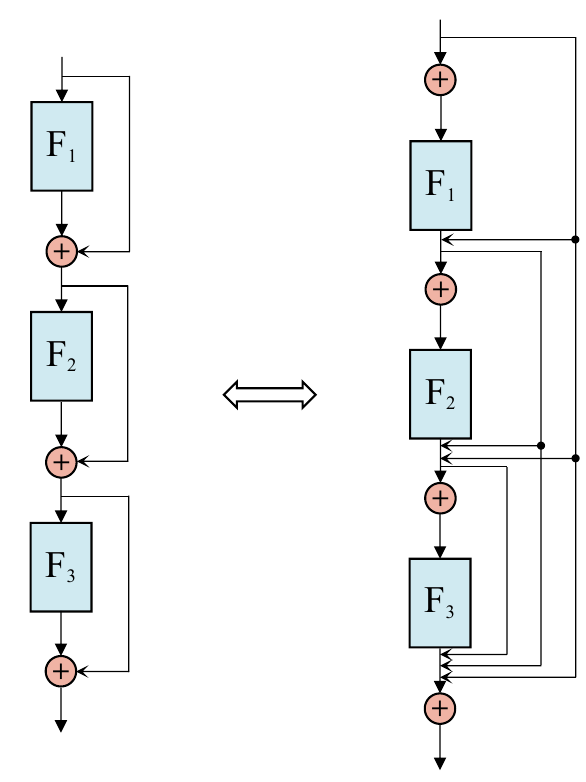}
  }
  \hspace{1in}
  \subfloat[]{
  \label{extension}
  \centering
  \includegraphics[height=1.0\columnwidth]{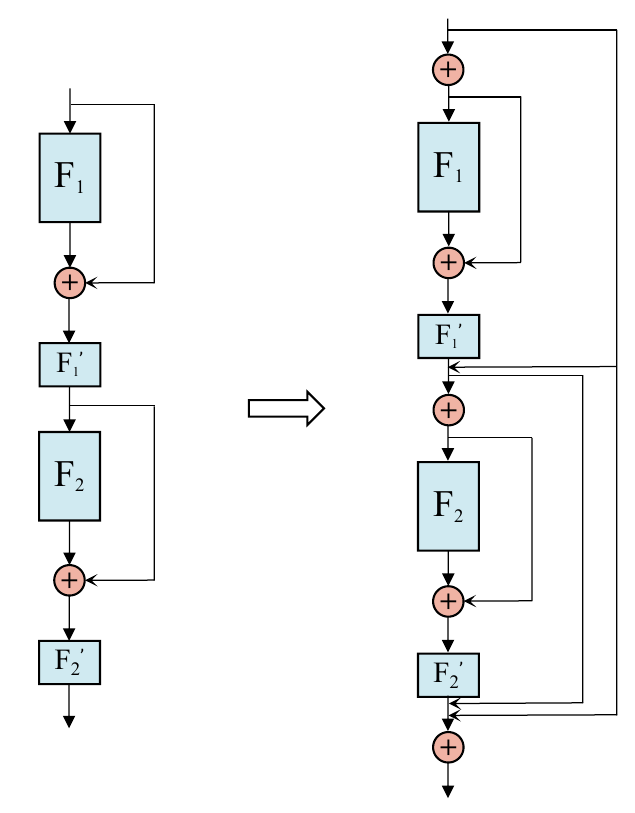}
  }
  \caption{Illustrations of Densely Additive Connection. Note that the structures in (a)
  are equivalent while those in (b) are not. (a) Abstraction. (b) Extension.}
  \label{fig4}
\end{figure*}

\begin{table}[t!]
  \centering
  
  \caption{Comparison between SEW, PA-A, and PA-B. (``PV'', ``FR'', ``Adv.'', and ``Dis.'' denote propagated value, firing rates,
  advantages, and disadvantages respectively.)}
  
  \label{table1}
  \resizebox{1.0\columnwidth}{!}{
      \begin{tabular}{cccc}
      \toprule

      & SEW & PA-A & PA-B \\
      
      \midrule
      PV  & Int & Analog & Analog \\
      FR & Medium & High & Low \\
      \midrule
      \multirow{4}*{Adv.} &  & the most active  & robust to \\ &
        robust to & flow of gradients & overfitting \\ &
        overfitting &  and delicate & and delicate\\ &
         &  perturbations & perturbations\\
      \multirow{2}*{Dis.} & coarse perturbations & overfitting on & sparse \\ &
      and sparse updates & small dataset & updates \\

      \bottomrule
    \end{tabular}
}
\end{table}

For the three improved connections, we will further compare them from two aspects, namely
the propagated values and structural characteristics as shown in Tab.~\ref{table1}, which determine the representational
and learning capability of the networks, respectively.
Firstly, for SEW, the values propagated through blocks are integers, which keep increasing monotonically
as the network deepens since each block only adds binary values to the propagated values.
For PA-A and PA-B, the values propagated through blocks are analog values. As the main idea
of residual learning is to learn small perturbations by each block, analog value is obviously a better choice
since it produces more delicate perturbations.
From this perspective, PA-A and PA-B have
stronger representational capability than SEW.
Besides, the structural characteristics of
SEW, PA-A, and PA-B also play an important role in the learning of
SNNs. Different from ANNs, as SNNs deepen,
the generated spikes become more and more sparse, which hinders the learning of
large-scale SNNs according to {\bf Proposition~\ref{proposition2}}.
But this is not the case for PA-A according to {\bf Proposition~\ref{proposition3}}.

\begin{proposition}
    \label{proposition2}
    The intrinsic sparsity of SNNs tends to vanish the gradients and hinder the learning of SNNs.
\end{proposition}
\begin{proof}
  Assuming that we are updating the synapse weight matrix $W^l$
  of the SNs in $l$-th layer, we have:
  \begin{equation}
      \label{eq5}
      \begin{gathered}
      \begin{aligned}
          \frac{\partial L}{\partial W^l} & = L_{spatial} + L_{temporal},
      \end{aligned}
      \end{gathered}
  \end{equation}
  \begin{equation}
    \label{eq7}
    \begin{gathered}
    \begin{aligned}
      L_{spatial} = \frac{\partial L}{\partial S^{l + 1}_t}
        \cdot \frac{\partial S^{l + 1}_t}{\partial V^l_t}
        \cdot \frac{\partial V^l_t}{\partial W^l},
    \end{aligned}
    \end{gathered}
\end{equation}
\begin{equation}
  \label{eq8}
  \begin{gathered}
  \begin{aligned}
      L_{temporal} = \frac{\partial L}{\partial U^l_t}
      \cdot \frac{\partial U^l_t}{\partial V^l_t}
      \cdot \frac{\partial V^l_t}{\partial W^l},
  \end{aligned}
  \end{gathered}
\end{equation}
\begin{equation}
  \label{eq9}
  \begin{gathered}
  \begin{aligned}
    \frac{\partial S^{l+1}_t}{\partial V^l_t} = \bar{\Theta}'(V^l_t - V_{th}),
  \end{aligned}
  \end{gathered}
\end{equation}
  where $L_{spatial}$, $L_{temporal}$, and $L$ denote the loss propagating through layers, the loss propagating
  through time and the final loss, respectively. For clarity, we illustrate the computational
  graph in Fig.~\ref{fig2}. We adopt $\bar{T}$ and $\bar{L}$ to denote the total number of time steps and layers.
  Note that we assume $0<t\leq\bar{T}$ and $0<l<\bar{L}$.
  Propagating through time and layers, we have:
  \begin{equation}
      \label{eq6}
      \begin{gathered}
        \frac{\partial L}{\partial U^l_{t-1}} = 
        \left\{
          \begin{aligned}
              \frac{\partial L}{\partial S^{l + 1}_{t}}
              \cdot \frac{\partial S^{l+1}_{t}}{\partial V^l_{t}}
              \cdot \frac{\partial V^l_{t}}{\partial U^l_{t-1}}
              \\ + \frac{\partial L}{\partial U^l_{t}}
              \cdot \frac{\partial U^l_{t}}{\partial V^l_{t}}
              \cdot \frac{\partial V^l_{t}}{\partial U^l_{t-1}}, & & 0 < t < \bar{T},\\
              \frac{\partial L}{\partial S^{l + 1}_{t}}
              \cdot \frac{\partial S^{l+1}_{t}}{\partial V^l_{t}}
              \cdot \frac{\partial V^l_{t}}{\partial U^l_{t-1}}, & & t = \bar{T},
          \end{aligned}\right.
      \end{gathered}
  \end{equation}
  \begin{equation}
    \label{eq10}
    \begin{gathered}
      \frac{\partial L}{\partial S^{l}_t} = 
      \left\{
        \begin{aligned}
            \frac{\partial L}{\partial S^{l + 1}_{t}}
            \cdot \frac{\partial S^{l+1}_{t}}{\partial V^l_{t}}
            \cdot \frac{\partial V^l_{t}}{\partial S^l_t}
            \\ + \frac{\partial L}{\partial U^l_{t}}
            \cdot \frac{\partial U^l_{t}}{\partial V^l_{t}}
            \cdot \frac{\partial V^l_{t}}{\partial S^l_t}, & & 0 < t < \bar{T},\\
            \frac{\partial L}{\partial S^{l + 1}_{t}}
            \cdot \frac{\partial S^{l+1}_{t}}{\partial V^l_{t}}
            \cdot \frac{\partial V^l_{t}}{\partial S^l_t}, & & t = \bar{T}.
        \end{aligned}\right.
    \end{gathered}
\end{equation}
  Substituting Eq.~\eqref{eq6} and Eq.~\eqref{eq10} into Eq.~\eqref{eq7} and Eq.~\eqref{eq8}
  iteratively, every item is multiplied by Eq.~\eqref{eq9}.
  When the gradients propagate through layers, Eq.~\eqref{eq7} and Eq.~\eqref{eq8} will be multiplied
  by Eq.~\eqref{eq9} for multiple times. And the intrinsic sparsity in SNNs makes Eq.~\eqref{eq9} tend to
  produce small values.
  Thus, Eq.~\eqref{eq5} will tend to converge to 0.
  That is to say, although the sparsity makes SNNs energy-efficient, it also
  vanishes the gradients and hinders the learning of SNNs.
\end{proof}
\begin{proposition}
    \label{proposition3}
    For PA-A, within a stage, the firing rates of the first SN in a deeper block tend to
    be higher.
\end{proposition}
\begin{proof}
Since each block ends with a BN layer,
the first SN in the $l$-th block from stage-$i$ receives
the input current that can be approximated
by $N(\sum_{j=0}^{l-1}\beta_j, \sum_{j=0}^{l-1}\lambda_j^2)$,
where $\beta_j$ and $\lambda_j$ are the learnable bias and
scaling factor of the last BN layer in the $j$-th block.
As the network goes deeper, the growing value of
$\sum_{j=0}^{l-1}\lambda_j^2$ endows these SNs with higher firing rates.
\end{proof}

\begin{listing}[t!]%
  \caption{Densely Additive Connection}%
  \label{lst:listing}%
  \begin{lstlisting}[language=Python]
class DenseBlock(nn.Module):
def __init__(self, layers):
  super(DenseBlock, self).__init__()
  for i, l in enumerate(layers):
    self.add_module('layer%d' % (i + 1), l)
  self.layers = layers

def forward(self, x):
  features = [self.layers[0](x)]
  for layer in self.layers[1:]:
    new_feature = layer(sum(features[:]))
    features.append(new_feature)

  return sum(features[:])
  \end{lstlisting}
\end{listing}

In Fig.~\ref{fig3}, we compare the firing rates of SEW, PA-A, and PA-B.
We can observe that the behavior of PA-A is consistent with
{\bf Proposition~\ref{proposition3}}. The firing rates of its first SN in a residual block
keep growing as the network goes deeper until it reaches the starting downsample block of the next stage
and its overall firing rate is higher than others.
For SEW, since the values propagated increase
monotonically, it has higher firing rates than PA-B, but
such an effect is limited by the BN layer before SN.
From {\bf Proposition~\ref{proposition2}},
we can know that higher firing rates will facilitate the flow
of gradients. That is to say, PA-A has the most active flow of gradients.
But it also makes PA-A suffer
from overfitting when not given enough data
and obtain poor results on small datasets.
And for SEW and PA-B, when the dataset is small,
sparse updates act as the role of regularization
and enable them to obtain better performance.
And the more delicate perturbations of PA-B further endow it with
better performance than SEW on small datasets.
However, when it comes to large-scale dataset,
facilitating the flow of gradients becomes the key,
instead of avoiding overfitting, and PA-A manages to obtain
substantially higher accuracy as we will demonstrate in Sec.~\ref{ablation}.
Since MS-ResNet adopts tdBN,
{\bf Proposition~\ref{proposition3}} is not available for it and ultimately it
obtains inferior performance to PA-A. The merits and demerits of these three
connections are summarized in Tab.~\ref{table1}.

\subsection{Densely Additive Connection}
\label{DA}
The performance gap between the connections with/without interblock SNs is so huge that limits the
applications of the connections with interblock SNs. But the diversity of structures is
important for the development of SNNs.
Thus, it's necessary to find a route to improve the performance
of the connections with interblock SNs.

The key problem in these connections is that the information flow is blocked by interblock SNs.
For the connections without interblock SNs,
the perturbations produced by each block can reach subsequent
blocks directly. To be specific, we denote the $i$-th residual block as $F_i$, the input to $F_i$ as $\dot{F_i}$, and the output of
$F_i$ as $\bar{F_i}$. Assuming that there are four blocks $F_0$, $F_1$, $F_2$, and $F_3$, for connections
without interblock modules, we have:
\begin{equation}
    \label{eq:1}
    \begin{aligned}
        \dot{F_1} = \bar{F_0},\qquad \dot{F_2} = \bar{F_0} + \bar{F_1},\qquad \dot{F_3} = \bar{F_0} + \bar{F_1} + \bar{F_2}.
    \end{aligned}
\end{equation}
The input to each block is the sum of the output from previous blocks.
That is to say, the output from each block can reach the subsequent
blocks without information loss and so as the gradients.
Following Eq.~\eqref{eq:1},
we can abstract these connections
into densely additive (DA) connection,
as illustrated in Fig.~\subref*{inspiration}
(note that $F_0$ is not included in the figure).

According to {\bf Proposition~\ref{proposition1}},
the gradient issue in these connections is caused by the interblock modules.
When we extend the DA connection to these connections, we manage to
construct a cleaner path for gradients to flow and alleviate
the negative effect brought by interblock modules.
Specifically, assuming that there are four blocks, for the $i$-th block, we denote the sub-block
between identity operation as $F_i$ and
the sub-block outside identity operation as $F_i'$.
As illustrated in Fig.~\subref*{extension}
(note that $F_0$ and $F_3$ are not included in the figure), after transformation,
for the input to each block, we have:
\begin{equation}
    \begin{gathered}
    \dot{F_1} = \bar{F_0'},\quad \dot{F_2} = \bar{F_1'},\quad \dot{F_3} = \bar{F_2'} \quad\Rightarrow \\
    \quad \dot{F_1} = \bar{F_0'},\quad \dot{F_2} = \bar{F_0'} + \bar{F_1'},\quad \dot{F_3} = \bar{F_0'} + \bar{F_1'} + \bar{F_2'}.
    \end{gathered}
\end{equation}
For clarity, we present torch-like Python code in Listing~\ref{lst:listing}.

\begin{proposition}
    \label{proposition4}
    Densely additive connection provides a lower bound for the gradients of the network.
\end{proposition}
\begin{proof}
  For DANet-C, assuming that we initialize the weight of
  the last BN layer in each block with 0, we have:
  \begin{equation}
      \label{eq11}
      \begin{gathered}
          O^{l + k}[t] = SN(\sum_{i=1}^{l+k-1}O^{i}[t]) = \\
          SN(\sum_{i=1}^{l+k-2}O^{i}[t] + SN(\sum_{i=1}^{l+k-2}O^{i}[t])) = \cdots,
          \\
          \frac{\partial O^{l + k}[t]}{\partial O^l[t]} =
           \sum_{i=1}^{k}
           \prod_{j=1}^{i}\bar{\Theta}'(O^{l+k-j+1} - V_{th}) \\
           > \bar{\Theta}'(O^{l+k} - V_{th}).
          \end{gathered}
  \end{equation}
  That is to say, with DA connection, we can find the shortest path for gradients to
  flow and avoid accumulative multiplication, which vanishes gradients.
  \end{proof}

We extend the DA connection to Origin and BAA and term the resultant
connections as DANet-C and DANet-D, respectively.
As PA-A and PA-B can also be abstracted as DANet,
we term them DANet-A and DANet-B respectively.
Since the extension to tdBN will introduce a lot
of additional BN layers and BAA should have comparable performance to tdBN,
we do not apply DA connection to tdBN.

\section{Experiments}
\label{exp}
\subsection{Experimental Setup}
\label{setup}
\subsubsection{Dataset}
We empirically verify our analysis on ImageNet \cite{deng2009imagenet},
which is the most widely used
benchmark for evaluating large-scale networks.
ImageNet contains 1.28 million training images and 50000 validating images, consisting of 1000 classes.
Since we are studying
large-scale SNNs, we do not adopt other smaller static or neuromorphic
datasets (e.g. CIFAR10~\cite{krizhevsky2010cifar} or DVS-CIFAR10~\cite{li2017cifar10}). In order to evaluate performance on the datasets at different scales
and keep the results consistent and comparable, we randomly select some classes from ImageNet
to form new datasets. Specifically, we randomly select 5, 20, and 100 classes and
term the corresponding datasets as ``ImageNet-5'', ``ImageNet-20'', and
``ImageNet-100'' respectively. For clarity, the original ImageNet is termed
``ImageNet-1k''.

\begin{figure*}[t!]
  \centering
  \subfloat[]{
  \label{w_imagenet5}
  \centering
  \includegraphics[width=0.66\columnwidth]{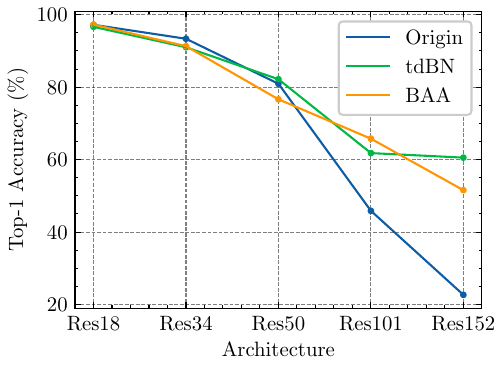}
  }
  \subfloat[]{
  \label{w_imagenet20}
  \centering
  \includegraphics[width=0.66\columnwidth]{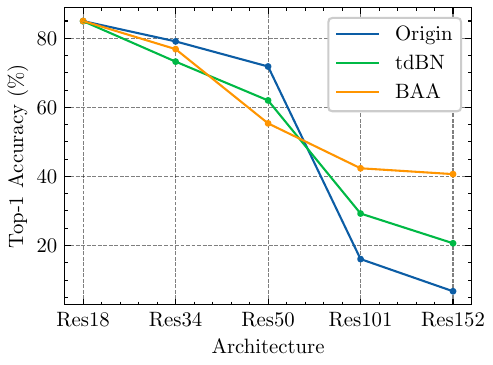}
  }
  \subfloat[]{
  \label{w_imagenet100}
  \centering
  \includegraphics[width=0.66\columnwidth]{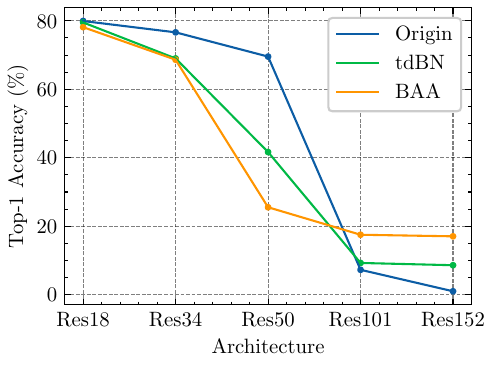}
  }
  \caption{Performance Comparison between Origin, tdBN, and BAA. (a) ImageNet-5. (b) ImageNet-20. (c) ImageNet-100.}
  \label{fig_w}
\end{figure*}

\begin{figure*}[t!]
    \centering
    \subfloat[]{
    \label{wo_imagenet5}
    \centering
    \includegraphics[width=0.66\columnwidth]{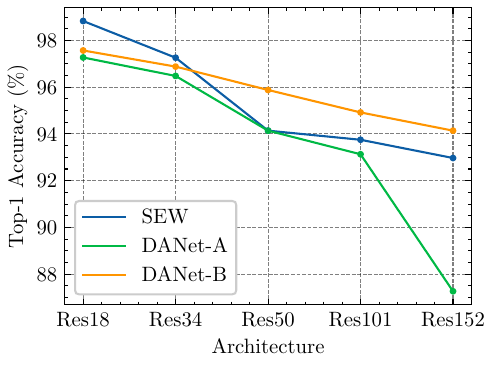}
    }
    \subfloat[]{
    \label{wo_imagenet20}
    \centering
    \includegraphics[width=0.66\columnwidth]{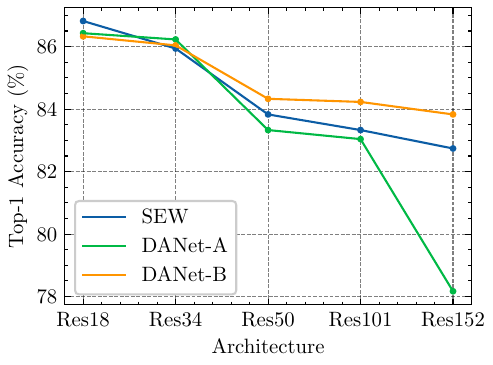}
    }

    \subfloat[]{
    \label{wo_imagenet100}
    \centering
    \includegraphics[width=0.66\columnwidth]{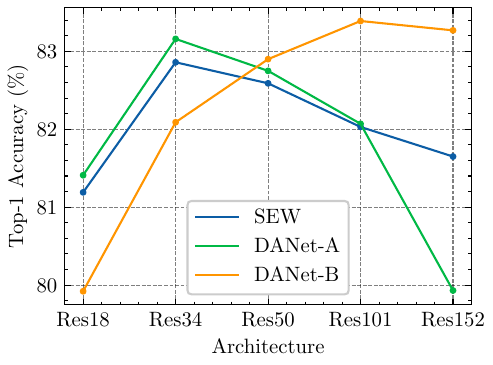}
    }
    \subfloat[]{
    \label{wo_imagenet1000}
    \centering
    \includegraphics[width=0.66\columnwidth]{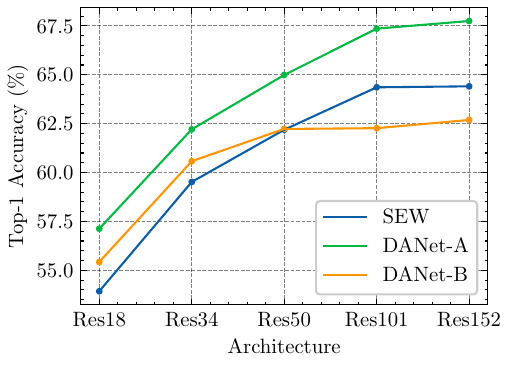}
    }
    \caption{Performance Comparison between SEW, DANet-A, and DANet-B. (a) ImageNet-5. (b) ImageNet-20. (c) ImageNet-100. (d) ImageNet-1k.}
    \label{fig_wo}
\end{figure*}

\subsubsection{ImageNet-1k}
The integrate and fire (IF) model is our default neuron model in the following experiments.
All the experiments are conducted with 8 Nvidia RTX 3090 GPUs using SpikingJelly~\cite{SpikingJelly}, which is based on
Pytorch~\cite{paszke2019pytorch}, with
mixed precision training \cite{micikevicius2018mixed} and sync-BN \cite{zhang2018context}.
We first randomly crop the images to
$224 \times 224$ and horizontally flip the images with a possibility of 0.5.
Auto-Augment \cite{Cubuk_2019_CVPR} and label-smoothing \cite{szegedy2016rethinking}
regularization are adopted for further augmentation.
We train our models using SGD with a momentum of 0.9.
For ablation studies, we warm up for 5 epochs with
a total of 50 epochs. The
learning rate is divided by 10 at epoch \{30, 40, 45\}. 
For performance comparison with state-of-the-art (SOTA) SNNs,
we warm up for 15 epochs with a total of 120 epochs and divide the learning rate by 10
at epoch \{60, 105, 115\}. Both of the above training scripts adopt 1 time step for accelerating training.
When the simulation time step is 4, we adopt the corresponding model trained with 1 time step as initialization,
train the models for another 45 epochs with 10 epochs of warming up, and divide the
learning rate by 10 at epoch \{30, 40\}.
The initial learning rate and weight decay are set to
0.1/0.02 per 256 batch size and 0.0001/0.00001, respectively, for the models with 1/4 simulation time step(s).
Since training cost is linearly related to the simulation time steps, such a two-phase training
method can effectively save training cost while obtaining considerable results.
Compared to SOTA SNNs
\cite{zheng2021going,fang2021deep} (e.g. 300 epochs),
we adopt relatively small epochs, which will result in suboptimal accuracy,
but we still manage to obtain satisfying results.

\subsubsection{ImageNet-5, ImageNet-20 and ImageNet-100}
On these subsets of ImageNet-1k, we adopt the same data augmentation strategy and optimizer
configuration as on ImageNet-1k.
All experiments are
conducted with a simulation time steps of 4 and an initial learning rate of 0.1 per 256 batch size
on 1 Nvidia RTX 3090 GPU.
For ImageNet-5 and ImageNet-20, we warm up for 5 epochs with a total of 100 epochs. And for
ImageNet-100, we warm up for 5 epochs with a total of 50 epochs. 
A cosine learning rate decay scheduler \cite{loshchilov2017sgdr} is adopted for all experiments.

Every experiment is conducted for 3 trials and the averaged top-1 accuracy is reported.
Training logs, pre-trained models, and code will be publicly available.

\subsection{Ablation Study of Residual Connection}
\label{ablation}
As shown in Fig.~\ref{fig_w} and Fig.~\ref{fig_wo},
we adopt five architectures to evaluate
the performance of different residual connections.
We can observe that, as the networks deepen, 
the performance of Origin, tdBN, and BAA consistently
decreases on ImageNet-5, ImageNet-20, and ImageNet-100.
Among them, the degradation of Origin is the worst, while
BAA and tdBN have a similar effect on stabilizing training
as we analyzed in Sec.~\ref{method:ImF}.
But compared to the other connections,
they suffer from intolerable degradation as the networks deepen.

For SEW, DANet-A, and DANet-B,
the results are also consistent with our analysis in Sec.~\ref{method:ImF}.
We can observe that when the datasets are small,
DANet-A performs much worse than the others. The
active flow of gradients makes DANet-A suffer from
overfitting when not given enough data.
As the datasets are getting larger,
the performance gap between DANet-A and the other connections
is becoming smaller and smaller,
and ultimately DANet-A becomes the best-performing connection.
For SEW and DANet-B, sparse updates make them robust to
overfitting and obtain better performance on small datasets.
And the more fine-grained perturbations of DANet-B endow it with
stronger representational capability and
better performance than SEW.
But on ImageNet-1k,
the sparsity hinders the flow of gradients and degrades their performance.
The experimental results are consistent
with Tab.~\ref{table1} and Fig.~\ref{fig3}, verifying our analysis.

\begin{figure*}[t!]
  \centering
  \subfloat[]{
  \label{da_imagenet5}
  \centering
  \includegraphics[width=0.66\columnwidth]{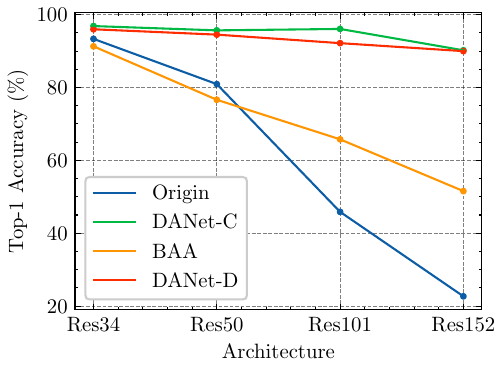}
  }
  \subfloat[]{
  \label{da_imagenet20}
  \centering
  \includegraphics[width=0.66\columnwidth]{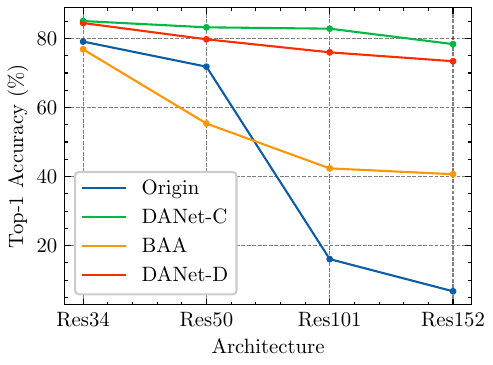}
  }
  \subfloat[]{
  \label{da_imagenet100}
  \centering
  \includegraphics[width=0.66\columnwidth]{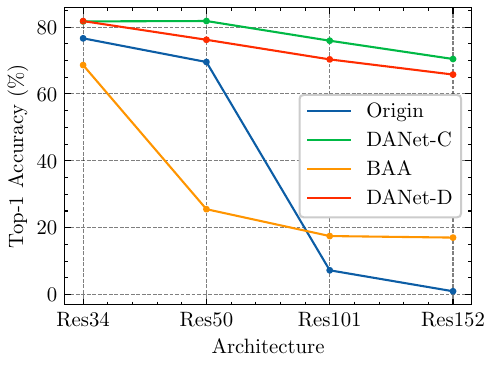}
  }
  \caption{Performance Comparison between Origin, DANet-C, BAA, and DANet-D on the subsets of ImageNet-1k. (a) ImageNet-5. (b) ImageNet-20. (c) ImageNet-100.}
  \label{fig_DA}
\end{figure*}

\begin{table}[t]
    \centering
        
    \caption{Performance Comparison between Origin, DANet-C, BAA, and DANet-D on ImageNet-1k.
    (``-'' means that such experiments are not conducted.)}
    \label{table2}
    \resizebox{0.9\columnwidth}{!}{
        \begin{tabular}{ccccccc}
        \toprule
    
        \multirow{2}{*}{\bf Architecture} & \multicolumn{4}{c}{\bf ImageNet-1k}\\
            
        & Origin & DANet-C & tdBN & DANet-D \\

        \midrule
        ResNet-34  & 61.86 & {\bf 68.26} & 63.72 & {\bf 65.50} \\
        ResNet-50  & 57.66 & {\bf 70.90} & 64.88 & - \\

        \bottomrule
        \end{tabular}
    }
\end{table}

\subsection{Ablation Study of Densely Additive Connection}
\label{ablation:DA}
As demonstrated in Fig.~\ref{fig_DA}, we adopt four architectures (without ResNet-18) to
verify the effectiveness of DA connection since every stage in ResNet-18 only has two blocks, {\em i.e.} the networks
with/without DA connection are the same.
On the subsets of ImageNet-1k,
DA connection greatly improves
performance and effectively alleviates degradation.
In Tab.~\ref{table2}, we compare DANet with corresponding architectures on ImageNet-1k.
Note that the comparison on ImageNet-1k is unfair for
DANet. For Origin, we adopt the results from~\cite{fang2021deep},
which are obtained with 300 epochs. And for tdBN, we also
adopt the results from~\cite{zheng2021going},
which are obtained with 6 time steps
and more training epochs.
We adopt the training recipe
used for our best-performing DANet-A, which is much cheaper.
Nevertheless, our DANet-C
outperforms Origin by up to 13.24\% and our DANet-D under ResNet-34
setting outperforms tdBN under all settings with less training cost.
We do not further increase their depth
since the results are sufficient to demonstrate the
superiority of DA connection.

\begin{table*}[t]
    \centering
    \caption{Performance Comparison between different dense operations. (``-'' means that such experiments are not conducted.)}
    \label{table3}
    \resizebox{1.6\columnwidth}{!}{
        \begin{tabular}{ccccccccccccc}
        \toprule

        \multirow{2}{*}{\bf Architecture} & \multicolumn{2}{c}{\bf ImageNet-5} & \multicolumn{2}{c}{\bf ImageNet-20} & \multicolumn{2}{c}{\bf ImageNet-100} & \multicolumn{2}{c}{\bf ImageNet-1k} \\
        
        & Add & Concate & Add & Concate & Add & Concate & Add & Concate \\
        
        \midrule

        ResNet-18/DenseNet-121  & {\bf 97.22} & 94.92 & {\bf 84.98} & 83.30 & {\bf 79.94} & 79.82 & - & 47.78 \\
        ResNet-34/DenseNet-201  & {\bf 96.88} & 95.70 & {\bf 85.06} & 84.04 & 81.69 & {\bf 81.88} & {\bf 57.18} & 49.89 \\
        ResNet-50/DenseNet-161  & 95.70 & {\bf 97.66} & 83.23 & {\bf 85.32} & 81.81 & {\bf 84.16} & {\bf 61.79} & - \\

        \bottomrule
    \end{tabular}
    }
    
\end{table*}

\begin{figure*}[t!]
  \centering

  \subfloat[]{
  \label{gs_res101}
  \centering
  \includegraphics[width=0.66\columnwidth]{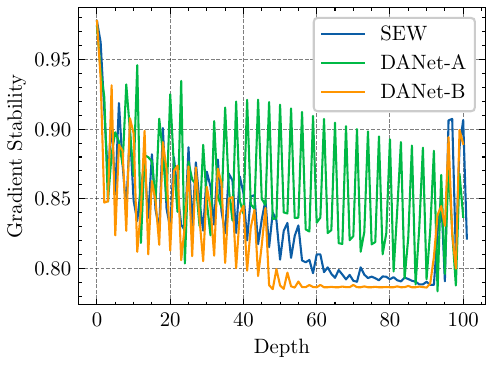}
  }
  \subfloat[]{
  \label{gs_res152}
  \centering
  \includegraphics[width=0.66\columnwidth]{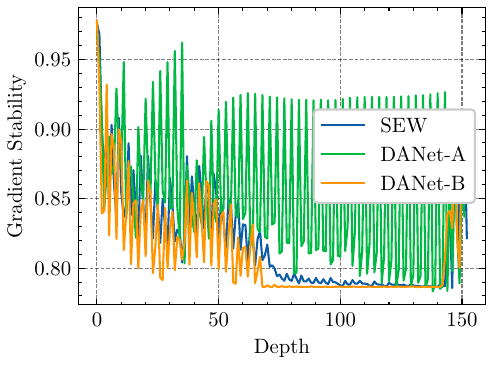}
  }
  \subfloat[]{
  \label{overall_gs}
  \centering
  \includegraphics[width=0.66\columnwidth]{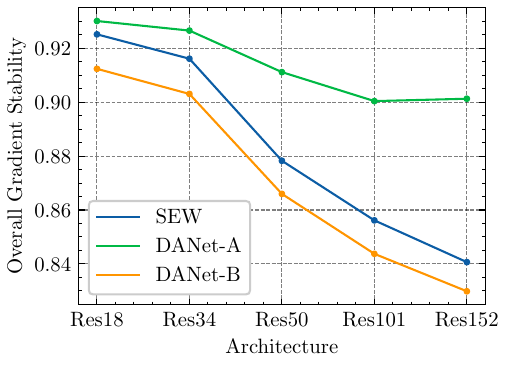}
  }

  
  \caption{Gradient stability of SEW, DANet-A, and DANet-B. (a)-(b) depict the layer-wise gradient stability of different
  backbones and (c) demonstrates the overall gradient stability. (a) ResNet-101. (b) ResNet-152. (c) Overall gradient stability.}
  \label{fig8}
\end{figure*}



\begin{figure*}[t]
    \centering
  \subfloat[]{
  \centering
  \includegraphics[width=0.66\columnwidth]{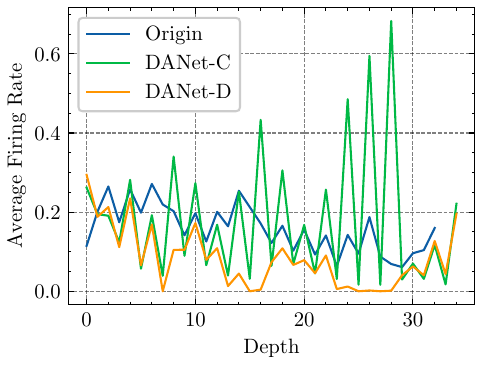}
  }
  \subfloat[]{
  \centering
  \includegraphics[width=0.66\columnwidth]{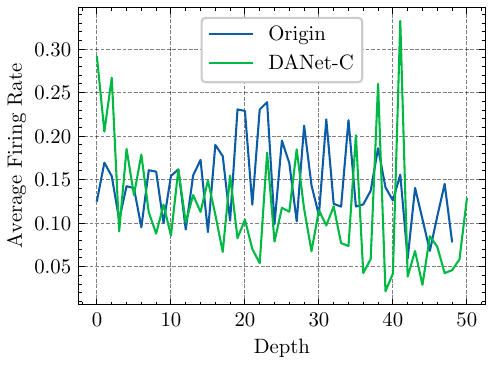}
  }
  \caption{Layer-wise firing rates of Origin, DANet-C, and DANet-D. (a) ResNet-34. (b) ResNet-50.}
  \label{fig10}
\end{figure*}

\begin{figure*}[t]
    \centering
  \subfloat[]{
  \centering
  \includegraphics[width=0.66\columnwidth]{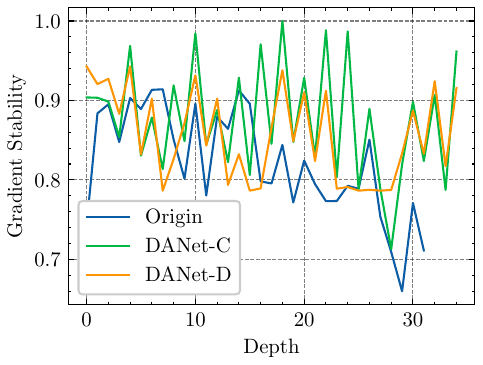}
  }
  \subfloat[]{
  \centering
  \includegraphics[width=0.66\columnwidth]{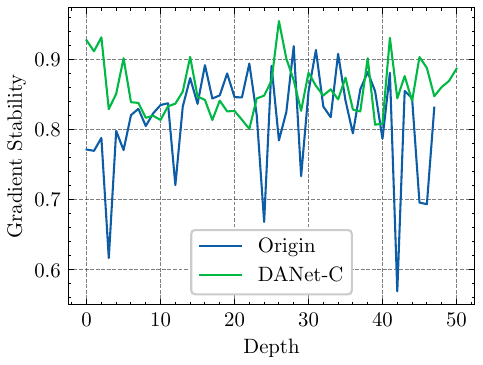}
  }
  \caption{Layer-wise gradient stability of Origin, DANet-C, and DANet-D. (a) ResNet-34. (b) ResNet-50.}
  \label{fig11}
\end{figure*}

\begin{figure*}[t]
  \centering
  \subfloat[]{
  \label{fig14:Ori}
  \includegraphics[width=0.50\columnwidth]{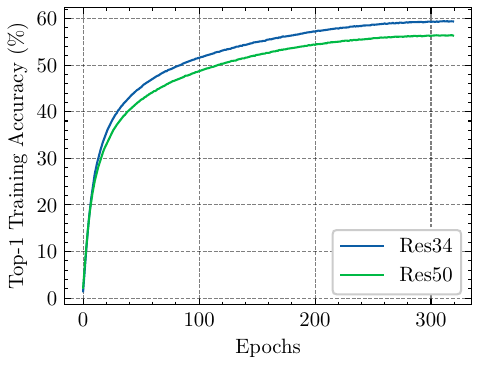}
  \includegraphics[width=0.50\columnwidth]{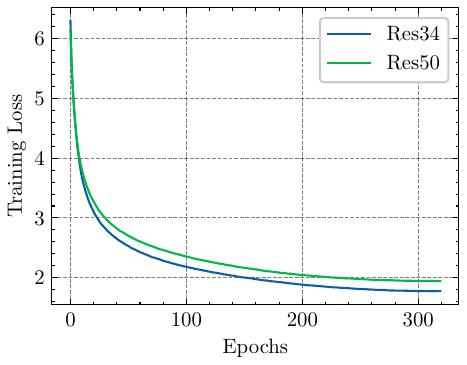}
  }
  \subfloat[]{
  \label{fig14:DAC}
  \includegraphics[width=0.50\columnwidth]{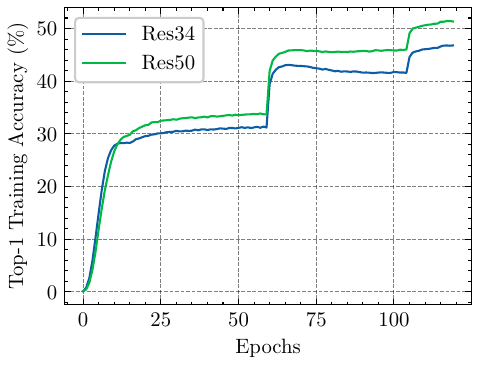}
  \includegraphics[width=0.50\columnwidth]{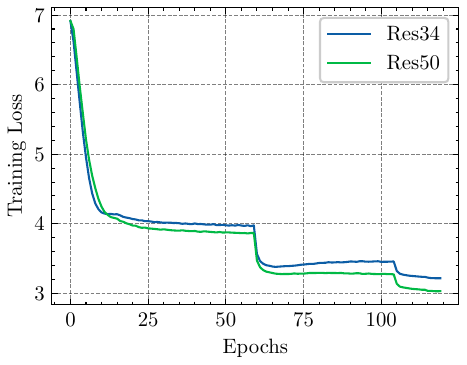}
  }
  \caption{Training loss and top-1 training accuracy of Origin and DANet-C on ImageNet-1k. (a) Origin. (b) DANet-C.}
  \label{fig12}
\end{figure*}

  
          
          
  
  
\subsection{Ablation Study of Dense Operation}
\label{ablation:DO}
As demonstrated in Tab.~\ref{table3}, we adopt
three pairs of architectures to compare the
effect of different dense operations in training large-scale
SNNs. For concatenation operation, we adopt the architectures
from~\cite{huang2017densely} that have comparable parameters
to ResNet-18/34/50 and modify them to their SNN versions.
We can observe that the concatenation operation performs well
on smaller datasets, but on large-scale dataset,
the concated sparse binary feature maps limit the representational
capability of the networks and result in their
poor performance on ImageNet-1k.
Although concatenation obtains considerable performance
in ANNs, the unique features of SNNs
results in the final huge performance gap (7.29\%)
on ImageNet-1k, which indicates that
dense addition is a better choice for training large-scale
SNNs. Since we are studying large-scale SNNs in this paper,
we adopt addition as our default dense operation.
Nevertheless, we conjecture that the sparse feature maps also act as a role of regularization.
It enables the performance of resultant networks to increase as networks get larger even on small datasets,
which is opposite to the additive networks and makes concatenation a promising choice on small datasets.

\subsection{Quantitative Analysis on Densely Additive Connection}

\begin{table}[t]
  \centering
  \caption{Quantitative comparison between SEW, DANet-A, and DANet-B on gradient stability.}
  \label{table5}
  \resizebox{0.9\columnwidth}{!}{
      \begin{tabular}{cccc}
      \toprule

      \multirow{2}{*}{\bf Architecture} & \multicolumn{3}{c}{\bf Gradient Stability} \\
      
      & SEW & DANet-A & DANet-B  \\
      
      \midrule

      ResNet-18  & {\bf 0.0951} & 0.0902 & 0.0911 \\
      ResNet-34  & 0.0075 & {\bf 0.0082} & 0.0051 \\
      ResNet-50  & 0.0007 & {\bf 0.0020} & 0.0005 \\
      ResNet-101 & $7.5\times 10^{-9}$ & {\boldmath $2.5\times 10^{-7}$} & $2.2\times 10^{-9}$ \\
      ResNet-152 & $5.2\times 10^{-14}$ & {\boldmath $4.1\times 10^{-11}$} & $1.3\times 10^{-14}$ \\

      \bottomrule
  \end{tabular}
  }
\end{table}

\begin{table}[t]
  \centering
  \caption{Quantitative comparison between Origin, DANet-C, and DANet-D on gradient stability.}
  \label{table6}
  \resizebox{0.9\columnwidth}{!}{
      \begin{tabular}{cccc}
      \toprule

      \multirow{2}{*}{\bf Architecture} & \multicolumn{3}{c}{\bf Gradient Stability} \\
      
      & Origin & DANet-C & DANet-D  \\
      
      \midrule

      ResNet-34  & $1.2\times 10^{-5}$ & {\bf 0.0088} & 0.0040 \\
      ResNet-50  & $1.0\times 10^{-6}$ & {\bf 0.0004} & - \\

      \bottomrule
  \end{tabular}
  }
\end{table}

\begin{table*}[t]
  \centering
  \caption{Comparison with SOTA SNN methods on ImageNet-1k dataset.}
  \label{table4}
  \resizebox{1.5\columnwidth}{!}{
  \begin{tabular}{cccccc}
      \toprule
      Arch. & Connection & Method & T & Par.(M) & Accuracy(\%) \\

      \midrule
      \multirow{6}*{ResNet-18} & 
      DANet-B\cite{meng2022training} & DSR & 50 & \multirow{6}*{11.69} & \textbf{67.74} \\ &
      tdBN\cite{guo2022reducing} & \multirow{5}*{Directly training} & 4 &  & 64.78 \\ &
      SEW\cite{fang2021deep} &  & 4 &  & 63.18 \\ &
      SEW\cite{10.1007/978-3-031-19775-8_4} &  & 4 &  & 63.68 \\ &
      MS\cite{hu2021advancing} &  & 6 &  & 63.10 \\ & 
      \textbf{DANet-A (this work)} &  & \textbf{4} &  & \textbf{66.48} \\
      
      \midrule
      \multirow{12}*{ResNet-34} & 
      \multirow{2}*{tdBN\cite{zheng2021going}} & \multirow{12}*{Directly training} & 6 & 86.13 & 67.05 \\ &
       &  & 6 & \multirow{11}*{21.80} & 63.72 \\ &
      tdBN\cite{guo2022reducing} &  & 4 &  & 65.54 \\ &
      tdBN\cite{guo2022recdis} &  & 6 &  & 67.33 \\ &
      tdBN\cite{NEURIPS2022_010c5ba0} &  & 6 &  & 67.43 \\ &
      tdBN\cite{li2021differentiable} &  & 6 &  & 68.19 \\ &
      SEW\cite{fang2021deep} &  & 4 &  & 67.04 \\ &
      SEW\cite{10.1007/978-3-031-19775-8_4} &  & 4 &  & 67.69 \\ &
      MS\cite{hu2021advancing} &  & 6 &  & 69.42 \\ &
      \textbf{DANet-A (this work)} &  & \textbf{4} &  & \textbf{71.22} \\ &
      \textbf{DANet-C (this work)} &  & \textbf{4} &  & \textbf{68.26} \\ &
      DANet-D (this work) &  & 4 &  & 65.50 \\ 

      \midrule
      \multirow{4}*{ResNet-50} & 
      SEW\cite{fang2021deep} & \multirow{4}*{Directly training} & 4 & \multirow{4}*{25.56} & 67.78 \\ &
      \multirow{2}*{\textbf{DANet-A(this work)}} &  & \textbf{1} &  & \textbf{67.70} \\ &
       &  & \textbf{4} &  & \textbf{73.71} \\ &
      \textbf{DANet-C(this work)} &  & \textbf{4} &  & \textbf{70.90} \\

      \midrule
      ResNet-104 &
      MS\cite{hu2021advancing} & \multirow{4}*{Directly training} & 5 & 77.35 & 74.21 \\
      \multirow{3}*{ResNet-101} & 
      SEW\cite{fang2021deep} &  & 4 & \multirow{3}*{44.55} & 68.76 \\ &
      \multirow{2}*{\textbf{DANet-A(this work)}} &  & \textbf{1} &  & \textbf{70.64} \\ &
      &  & \textbf{4} &  & \textbf{76.13} \\

      \midrule
      \multirow{3}*{ResNet-152} & 
      SEW\cite{fang2021deep} & \multirow{3}*{Directly training} & 4 & \multirow{3}*{60.19} & 69.26 \\ &
      \multirow{2}*{\textbf{DANet-A(this work)}} &  & \textbf{1} &  & \textbf{71.62} \\ &
       &  & \textbf{4} &  & \textbf{77.22} \\

      \bottomrule
  \end{tabular}
  }
\end{table*}

\subsubsection{DANet-A/B}
\label{qa:AB}
To further demonstrate our analysis of the firing rates'
influence on gradients and the property of the discussed residual
connections, we intuitively visualize the gradient
stability of SEW, DANet-A, and DANet-B in Fig.~\ref{fig8}.
We denote the output and the membrane potential before spike generation
of the SNs in $l$-th layers as $O^l$ and $V^l$.
Based on the analysis in {\bf Proposition~\ref{proposition2}},
we can know that $\frac{\partial O^l}{\partial V^l}$
is the most representative component to illustrate the stability of
gradients. Thus here we take its amplitude
on behalf of gradient stability. Note that the closer the amplitude is to 1,
the more stable the gradients.

We can observe that the results are consistent with the results of firing rates.
DANet-A has much more stable gradients than SEW and
DANet-B and exhibits a clear pattern
as in firing rates.
The deeper the networks,
the clearer the differences.
To illustrate this idea more intuitively, we present further
quantitative comparison in Tab.~\ref{table5}.
Since backpropagation adopts the chain rule to propagate gradients,
applying the chain rule to gradient stability is an intuitive way to
represent the overall stability of the networks
quantitatively, and that is what we do in Tab.~\ref{table5}.
We can find that, as the increase of depth,
the gap between SEW and DANet-A is becoming larger and larger. When it
comes to ResNet-152, the accumulated result of DANet-A
is three orders of magnitude higher than that of SEW and DANet-B.
These properties bring strong learning capability and
better performance for DANet-A

\subsubsection{DANet-C/D}
\label{qa:CD}
In Fig.~\ref{fig10}, we compare the firing rates of Origin, DANet-C, and DANet-D. We can find that, DANet-C has adequate
firing rates for training large-scale SNNs, while the additional interblock BN layer in DANet-D alleviates the effect of
DA connection and ultimately results in suboptimal performance on ImageNet-1k. Note that although Origin has higher firing
rates than DANet-D, the propagated gradients in DANet-C and DANet-D manage to reach shallower layers directly through
DA connection, which avoids vanishing gradients. As shown in Fig.~\ref{fig11}, DANet-C and DANet-D have much more stable
gradients than Origin. And similar to Sec.~\ref{qa:AB}, we further quantitatively compare the gradient stability between
Origin, DANet-C, and DANet-D in Tab.~\ref{table6}. DA connection enables DANet-C and DANet-D to obtain two orders
of magnitude higher results than Origin and also solves the degradation problem in Origin. As shown in
Fig.~\subref*{fig14:Ori}, as the increase of depth, the training accuracy of Origin decreases and the training loss increases,
which demonstrates that deeper networks have difficulty in learning. But in Fig.~\subref*{fig14:DAC}, we can find that the
degradation problem is well solved by DA connection. Deeper networks with stronger representational capability manage
to obtain higher training accuracy and lower training loss, indicating that the learning issue in Origin is
effectively alleviated by DA connection.

\subsection{Comparison with State-of-the-Art SNNs}
\label{SOTA}
In Tab.~\ref{table4}, we compare our results with SOTA SNN methods on ImageNet-1k.
At a low training expense,
we manage to outperform other works by a large margin under all settings,
except \cite{meng2022training}, which spends 15 times our training cost.
Notably, we outperform SEW-ResNet using about a quarter of their training cost
by a large margin under all settings.
When the depth reaches 50 layers, we can achieve comparable or even better
performance than SEW-ResNet with only 1 time step and one-tenth of their training cost.
The performance gap keeps growing with the
increasing depth and our DANet-A manages to outperform SEW-ResNet by up to 7.96\%.
Moreover, our DANet-C also manages to obtain
competitive performance and outperforms other works under all settings,
further demonstrating the effectiveness of DA connection.
Although MS-ResNet has a similar topology as our DANet-A, the adopted tdBN
makes {\bf Proposition~\ref{proposition3}} no longer valid and greatly changes
its behavior. Under the ResNet-18 setting, MS obtains even lower accuracy than SEW
with 1.5 times their time steps. While we outperform MS-ResNet under all settings
by a considerable margin with less than half of their training cost. When it comes
to ResNet-101/104, our DANet-A outperforms MS-ResNet with a much smaller model
and a simpler training recipe.

\section{Conclusion and Discussion}
\subsection{Conclusion}
In this paper, we analyze six different residual connections
and substantiate our findings empirically through extensive experiments.
Then, building upon our observations, we abstract the best-performing residual connections
into densely additive (DA) connection and extend it to other topologies,
culminating in the introduction of the DANet family.
Our DANet variants exhibit superior performance on the ImageNet dataset,
all the while incurring a minor training cost.
Additionally, in order to provide an elaborate methodology for designing the topology of
large-scale SNNs, we expound upon the
applicable scenarios for different DANets and elucidate their advantages over preexisting methods.
We anticipate that this work will offer valuable insights for future works
to design the topology of their networks and, consequently, foster the advancement of large-scale SNNs.

\subsection{Discussion}
In this study, we mainly focus on the topology of large-scale SNNs
and base our architectures on the ResNet family from ANNs ({\em e.g.} ResNet-50).
But since the behavior of SNNs is quite different from ANNs,
prior intuition built upon ANNs is not available for SNNs and
new architectures specialized for SNNs
are worth studying, where our DA connection can also be applied.
Furthermore, our analysis has illuminated that even minor modifications to the network
topology exert a significant influence on the spiking patterns exhibited by SNNs.
These spiking patterns, in turn, play a pivotal role in shaping the learning process of SNNs.
Hence, there is considerable potential for the development of techniques geared towards rectifying firing rates,
ultimately facilitating the training of large-scale SNNs in a more precise and effective manner.

\section*{ACKNOWLEDGMENTS}
This research was funded by the STI 2030-Major Projects 2022ZD0208700, and the Beijing Municipal Natural Science Foundation under Grand No. 4212043.

{\small
\bibliographystyle{IEEEtranS}


}

\end{document}